\documentclass[11pt,a4paper]{article}
\usepackage[hyperref]{acl2020}
\usepackage{times}
\usepackage{latexsym}

\usepackage{microtype}

\aclfinalcopy 

\usepackage{amssymb}
\usepackage{amsmath}
\usepackage{graphicx}
\usepackage{multirow}
\usepackage[inline]{enumitem}
\usepackage{fancyhdr}
\usepackage{placeins}

\usepackage{caption}
\usepackage{subcaption}
\usepackage[symbol]{footmisc}
\usepackage{listings} 
\usepackage{amssymb}
\usepackage{algorithm}
\usepackage{algpseudocode}
\usepackage{amsthm,amsmath,amssymb}

\usepackage{mathrsfs}

\usepackage[normalem]{ulem}
\usepackage{cancel}

\usepackage{balance}  

\usepackage{microtype}
\DisableLigatures[f]{encoding=*, family =*}
\usepackage{comment}
\usepackage{mathtools}
\usepackage{bbm}
\usepackage{amsthm}
\usepackage{amsmath}
\usepackage{pifont}
\usepackage{threeparttable}  
\usepackage{booktabs}
\usepackage{multirow}
\usepackage{bm}
\usepackage{tablefootnote}

\newcommand{\cmark}{\text{\ding{51}}}
\newcommand{\xmark}{\text{\ding{55}}}

\newcommand{\SegVec}[1]{\mathbf{#1}}
\newcommand{\Outperform}[1]{\mathbf{#1}}
\newcommand{\TableSize}{\footnotesize} 

\newtheorem{proposition}{Proposition}

\title{SEEK: Segmented Embedding of Knowledge Graphs}

\author{Wentao  Xu$^1$, Shun Zheng$^2$, Liang He$^2$, Bin Shao$^2$, Jian Yin$^1$\thanks{\ \ Corresponding author.}, and Tie-Yan Liu$^2$\\
  $^1$ School of Data and Computer Science, Sun Yat-sen University, Guangzhou, China;\\
  Guangdong Key Laboratory of Big Data Analysis and Processing, Guangzhou, China\\
  $^2$ Microsoft Research Asia, Beijing, China \\
  \texttt{\{xuwt6@mail2,issjyin@mail\}.sysu.edu.cn} \\
  \texttt{\{Shun.Zheng,Liang.He,binshao,tyliu\}@microsoft.com} \\
}

\date{}

\begin{document}
\maketitle

\begin{abstract}
In recent years, knowledge graph embedding becomes a pretty hot research topic of artificial intelligence and plays increasingly vital roles in various downstream applications, such as recommendation and question answering.
However, existing methods for knowledge graph embedding can not make a proper trade-off between the model complexity and the model expressiveness, which makes them still far from satisfactory.
To mitigate this problem, we propose a lightweight modeling framework that can achieve highly competitive relational expressiveness without increasing the model complexity.
Our framework focuses on the design of scoring functions and highlights two critical characteristics:
1) facilitating sufficient feature interactions; 
2) preserving both symmetry and antisymmetry properties of relations.
It is noteworthy that owing to the general and elegant design of scoring functions, our framework can incorporate many famous existing methods as special cases.
Moreover, extensive experiments on public benchmarks demonstrate the efficiency and effectiveness of our framework.
Source codes and data can be found at \url{https://github.com/Wentao-Xu/SEEK}.
\end{abstract}

\section{Introduction}
\label{sec:intro}

Learning embeddings for a knowledge graph (KG) is a vital task in artificial intelligence (AI) and can benefit many downstream applications, such as personalized recommendation~\cite{zhang2016collaborative,wang2018dkn} and question answering~\cite{huang2019knowledge}.
In general, a KG stores a large collection of entities and inter-entity relations in a triple format, $(h, r, t)$, where $h$ denotes the \underline{h}ead entity, $t$ represents the \underline{t}ail entity, and $r$ corresponds to the relationship between $h$ and $t$. 
The goal of knowledge graph embedding (KGE) is to project massive interconnected triples into a low-dimensional space and preserve the initial semantic information at the same time.

Although recent years witnessed tremendous research efforts on the KGE problem, existing research did not make a proper trade-off between the model complexity (the number of parameters) and the model expressiveness (the performance in capturing semantic information).
To illustrate this issue, we categorize existing research into two categories.

The first category of methods prefers the simple model but suffers from poor expressiveness.
Some early KGE methods, such as TransE~\cite{bordes2013translating} and DistMult~\cite{yang2015embedding}, fell into this category.
It is easy to apply these methods to large-scale real-world KGs, but their performance in capturing semantic information (such as link prediction) is far from satisfactory.

In contrast, the second category pursues the excellent expressiveness but introduces much more model parameters and tensor computations.
Typical examples include TransH~\cite{wang2014knowledge}, TransR~\cite{lin2015learning}, TransD~\cite{ji2015knowledge}, Single DistMult~\cite{kadlec2017knowledge}, ConvE~\cite{dettmers2018conve} and InteractE~\cite{vashishth2019interacte}.
However, as pointed out by~\citet{dettmers2018conve}, the high model complexity often leads to poor scalability, which is prohibitive in practice because real-world KGs usually contain massive triples.

To address these drawbacks of existing methods, in this paper, we propose a light-weight framework for KGE that achieves highly competitive expressiveness without the sacrifice in the model complexity.
Next, we introduce our framework from three aspects: 
1) facilitating sufficient feature interactions, 
2) preserving various necessary relation properties, 
3) designing both efficient and effective scoring functions.

First, to pursue high expressiveness with the reasonable model complexity, we need to facilitate more sufficient feature interactions given the same number of parameters. 
Specifically, we divide the embedding dimension into multiple segments and encourage the interactions among different segments.
In this way, we can obtain highly expressive representations without increasing model parameters.
Accordingly, we name our framework as \underline{Se}gmented \underline{E}mbedding for \underline{K}Gs (SEEK). 

Second, it is crucial to preserve different relation properties, especially the symmetry and the antisymmetry.
We note that some previous research did not preserve the symmetry or the antisymmetry and thus obtained inferior performance~\cite{bordes2013translating,lin2015learning,yang2015embedding}.
Similar to the recent advanced models~\cite{trouillon2016complex,kazemi2018simple,sun2018rotate,DBLP:conf/acl/XuL19}, we also pay close attention to the modeling support of both symmetric and antisymmetric relationships.

Third, after an exhaustive review of the literature, we find that one critical difference between various KGE methods lies in the design of scoring functions.
Therefore, we dive deeply into designing powerful scoring functions for a triple $(h, r, t)$.
Specifically, we combine the above two aspects (facilitating feature interactions and preserving various relation properties) and develop four kinds of scoring functions progressively.
Based on these scoring functions, we can specify many existing KGE methods, including DistMult~\cite{yang2015embedding}, HoIE~\cite{nickel2016holographic}, and ComplEx~\cite{trouillon2016complex}, as special cases of SEEK.
Hence, as a general framework, SEEK can help readers to understand better the pros and cons of existing research as well as the relationship between them.
Moreover, extensive experiments demonstrate that SEEK can achieve either state-of-the-art or highly competitive performance on a variety of benchmarks for KGE compared with existing methods.

In summary, this paper makes the following contributions.
\begin{itemize}[label={-},leftmargin=1.5em]
	\item We propose a light-weight framework (SEEK) for KGE that achieves highly competitive expressiveness without the sacrifice in the model complexity.
	\item As a unique framework that focuses on designing scoring functions for KGE, SEEK combines two critical characteristics: facilitating sufficient feature interactions and preserving fundamental relation properties.
	\item As a general framework, SEEK can incorporate many previous methods as special cases, which can help readers to understand and compare existing research.
	\item Extensive experiments demonstrate the effectiveness and efficiency of SEEK. Moreover, sensitivity experiments about the number of segments also verify the robustness of SEEK. 
\end{itemize}

\section{Related Work}
\label{sec:related}

\begin{table*}[!h]
	\centering
	\TableSize
	\begin{tabular}{ l |c| c | c | c c}
		\toprule
		\multirow{2}{*}{\textbf{Methods}} & \multirow{2}{*}{\textbf{Scoring Function}} &  \multirow{2}{*}{\textbf{Performance}} & \multirow{2}{*}{\textbf{\# Parameters}} &\multicolumn{2}{c}{\textbf{Properties}} \\
		\cmidrule{5-6}
		& & & & \textbf{Sym}  & \textbf{Antisym}\\
		\midrule
		TransE~\cite{bordes2013translating} & $||\SegVec{h} + \SegVec{r} -\SegVec{t}||$ & Low & Small & \xmark & \cmark\\
		DistMult~\cite{yang2015embedding} & $\left< \SegVec{h}, \SegVec{r}, \SegVec{t} \right>$ & Low  & Small & \cmark & \xmark\\
		ComplEx~\cite{trouillon2016complex} & $Re(\left<\SegVec{h}, \SegVec{r}, \overline{\SegVec{t}} \right>)$ & Low   & Small  & \cmark & \cmark\\
		Single DistMult~\cite{kadlec2017knowledge} &$\left< \SegVec{h}, \SegVec{r}, \SegVec{t} \right>$ & High  & Large & \cmark & \xmark\\
		ConvE~\cite{dettmers2018conve} & $f(\rm{vec}(f([\SegVec{h},\SegVec{r}]*\omega))\rm{W})\SegVec{t}$& High  & Large & \xmark &  \cmark\\
		\midrule
		SEEK & $\sum s_{x, y} \left< \SegVec{r}_x, \SegVec{h}_y, \SegVec{t}_{w_{x, y}} \right>$  & High  & Small & \cmark & \cmark\\
		\bottomrule
	\end{tabular}
	\caption{Comparison between our SEEK framework and some representative knowledge graph embedding methods in the aspects of the scoring function, performance, the number of parameters, and the ability to preserve the symmetry and antisymmetry properties of relations.}	
	\label{tab:compare}
\end{table*}

We can categorize most of the existing work into two categories according to the model complexity and the model expressiveness.

The first category of methods is the simple but lack of expressiveness, which can easily scale to large knowledge graphs. This kind of methods includes TransE~\cite{bordes2013translating} and DistMult~\cite{yang2015embedding}. TransE uses relation $r$ as a translation from a head entity $h$ to a tail entity $t$ for calculating their embedding vectors of $(h, r, t)$; DistMult utilizes the multi-linear dot product as the scoring function.

The second kind of work introduces more parameters to improve the expressiveness of the simple methods. TransH~\cite{wang2014knowledge}, TransR~\cite{lin2015learning}, TransD~\cite{ji2015knowledge}, and ITransF~\cite{xie2017interpretable} are the extensions of TransE, which introduce other parameters to map the entities and relations to different semantic spaces. The Single DistMult~\cite{kadlec2017knowledge} increases the embedding size of the DistMult to obtain more expressive features. Besides, ProjE~\cite{shi2017proje}, ConvE~\cite{dettmers2018conve} and InteractE~\cite{vashishth2019interacte} leverage neural networks to capture more feature interactions between embeddings and thus improves the expressiveness. However, these neural network-based methods would also lead to more parameters since there are many parameters in the neural network. Although the second kind of methods has a better performance compared with simple methods, they are difficult to apply to real-world KGs due to the high model complexity (a large number of parameters).

Compared with the two types of methods above, our SEEK can achieve high expressiveness without increasing the number of model parameters. Table~\ref{tab:compare} shows the comparison between our framework and some representative KGE methods in different aspects. 

Besides, preserving the symmetry and antisymmetry properties of relations is vital for KGE models. Many recent methods devote to preserving these relation properties to improve the expressiveness of embeddings~\cite{trouillon2016complex,nickel2016holographic,guo2018:RUGE, boyang2018:aer,kazemi2018simple,sun2018rotate,DBLP:conf/acl/XuL19}.  Motivated by these methods, we also pay attention to preserving symmetry and antisymmetry properties of relations when we design our scoring functions.
\section{SEEK}
\label{sec:complex}


Briefly speaking, we build SEEK by designing scoring functions, which is one of the most critical components of various existing KGE methods, as discussed in the related work.
During the procedure of designing scoring functions, we progressively introduce two characteristics that hugely contribute to the model expressiveness:
1) facilitating sufficient feature interactions;
2) supporting both symmetric and antisymmetric relations.
In this way, SEEK enables the excellent model expressiveness given a light-weight model with the same number of parameters as some simple KGE counterparts, such as TransE~\cite{bordes2013translating} and DistMult~\cite{yang2015embedding}.

\subsection{Scoring Functions}
In this section, we illustrate our four scoring functions progressively.

\subsubsection{$\bm{f_1}$: Multi-linear Dot Product}
First, we start with the scoring function $f_1$ developed by~\citet{yang2015embedding}, which computes a multi-linear dot product of three vectors:
\begin{align}
f_1(h, r, t) = \left< \SegVec{r}, \SegVec{h}, \SegVec{t}\right> = \sum_{i} r_i \cdot h_i \cdot t_i,
\label{eq:dot-product}
\end{align}
where $\SegVec{r}, \SegVec{h}, \text{and } \SegVec{t}$ are low-dimensional representations of the relation $r$, the head entity $h$, and the tail entity $t$, respectively, and $r_i, h_i, \text{and } t_i$ correspond to the $i$-th dimension of $\SegVec{r}, \SegVec{h}, \text{and } \SegVec{t}$, respectively.

We note that the function $f_1$ is the building block of much previous research~\cite{trouillon2016complex,kadlec2017knowledge,kazemi2018simple}.
Different from these existing research, we focus on designing more advanced scoring functions with better expressiveness.

\subsubsection{$\bm{f_2}$: Multi-linear Dot Product Among Segments}

Next, we introduce fine-grained feature interactions to improve the model expressiveness further.
To be specific, we develop the scoring function $f_2$ that conducts the multi-linear dot product among different segments of the entity/relation embeddings.
First, we uniformly divide the $d$-dimensional embedding of the head $h$, the relation $r$, and the tail $t$ into $k$ segments, and the dimension of each segment is $d/k$.
For example, we can write the embedding of relation $\SegVec{r}$ as:
$$\SegVec{r} = [\SegVec{r}_0, \SegVec{r}_1, \ldots, \SegVec{r}_{k-1}], \quad \SegVec{r}_x \in \mathbb{R}^{d/k},$$
where $\SegVec{r}_x$ is the $x$-th segment of the embedding $\SegVec{r}$.

Then, we define the scoring function $f_2$ as follows:
\begin{align}
f_2(h, r, t) = \sum_{0 \leq x, y, w < k}  \left< \SegVec{r}_x, \SegVec{h}_y, \SegVec{t}_w \right>.
\label{eq:seg-score-function}
\end{align}
Compared with the scoring function $f_1$, where the interactions only happen among the same positions of $\SegVec{h}, \SegVec{r}, \text{and } \SegVec{t}$ embeddings, the scoring function $f_2$ can exploit more feature interactions among different segments of embeddings.

\begin{figure}[!h]
	\centering
	\includegraphics[width=1.0\columnwidth]{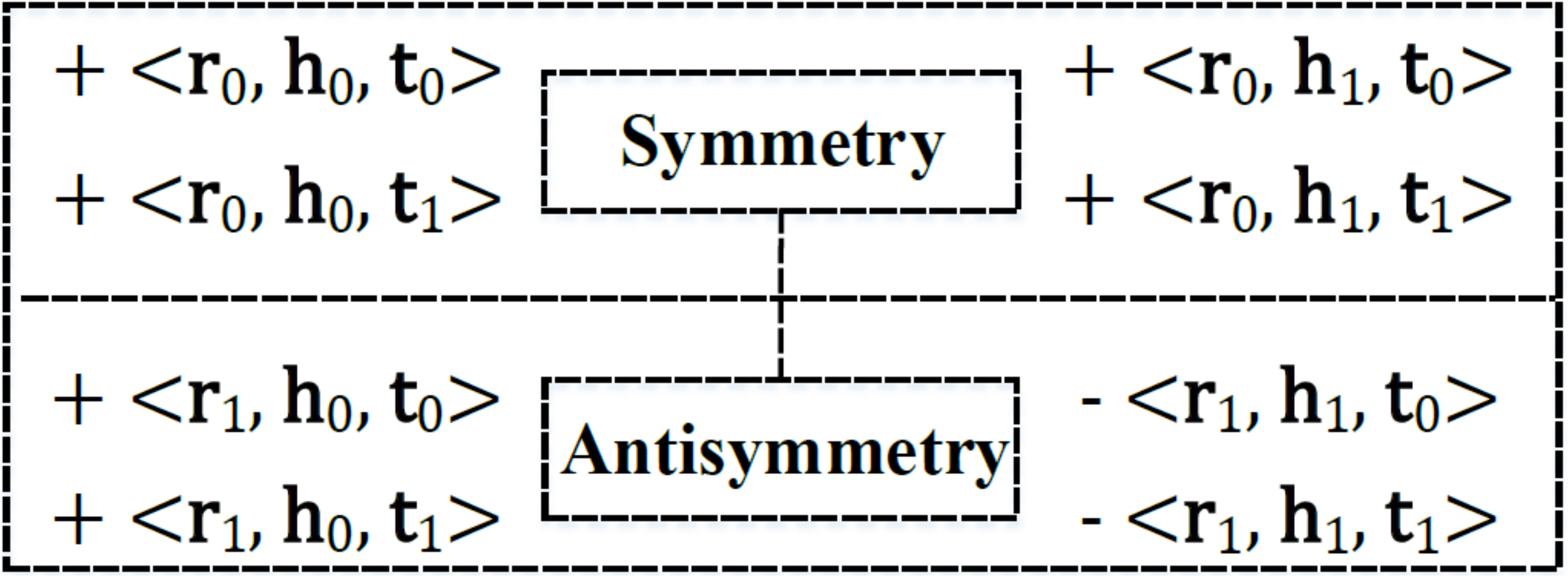}
	\caption{Scoring function $f_3$ with $k=2$.}
	\label{fig:f3_example}
\end{figure}
\subsubsection{$\bm{f_3}$: Modeling both Symmetric and Antisymmetric Relations}
Although the scoring function $f_2$ can facilitate fine-grained feature interactions, it can only preserve the symmetry property of relations and can not support the modeling of antisymmetric relations.
For example, given a symmetric relation $r$, we have $f_2(h, r, t) = f_2(t, r, h)$, but for an antisymmetric relation $r'$, the value of $f_2(h, r', t)$ is also equal to $f_2(t, r', h)$, which is unreasonable because $(t, r', h)$ is a false triple.

To preserve the antisymmetry property of relations, we divide the segments of relation embedding $\SegVec{r}$ into odd and even parts. Then we define a variable $s_{x, y}$ to enable the even parts of segments to capture the symmetry property of relations and the odd parts to capture the antisymmetry property. We define the scoring function after adding $s_{x, y}$ as:
\begin{align}
f_3(h, r, t) = \sum_{0 \leq x, y, w < k} s_{x, y} \cdot \left< \SegVec{r}_x, \SegVec{h}_y, \SegVec{t}_w \right>,
\label{eq:anti-score-function}
\end{align}
where
\[
s_{x, y} = 
\left\{
\begin{array}{cl}
-1, & \text{if } x \text{ is odd and } x + y \geq k, \\
1,  & \text{otherwise}. \\
\end{array}
\right.
\]
In the scoring function $f_3$, $s_{x, y}$ indicates the sign of each dot product term $\left< \SegVec{r}_x, \SegVec{h}_y, \SegVec{t}_w \right>$. 
Figure~\ref{fig:f3_example} depicts an example of the function $f_3$ with $k=2$.
When $\SegVec{r}_x$ is the even part of $\SegVec{r}$ (the index $x$ is even), $s_{x, y}$ is positive, and the summation $\sum_{s_{x, y} == 1} s_{x, y} \cdot \left< \SegVec{r}_x, \SegVec{h}_y, \SegVec{t}_w \right>$ of $f_3(h, r, t)$ equals to the corresponding one $\sum_{s_{x, y} == 1} s_{x, y} \cdot \left< \SegVec{r}_x, \SegVec{t}_y, \SegVec{h}_w \right>$ of $f_3(t, r, h)$.
Therefore, the function $f_3$ can model symmetric relations via the even segments of $\SegVec{r}$.
When $\SegVec{r}_x$ is the odd part of $r$ (the index $x$ is odd), $s_{x, y}$ can be either negative or positive depending on whether $x + y \geq k$.
Then, the summation of odd parts of $f_3(h, r, t)$ is differ from that of $f_3(t, r, h)$.
Accordingly, $f_3(h, r, t)$ can support antisymmetric relations with the odd segments of $\SegVec{r}$. 

\begin{figure}[!h]
	\centering
	\includegraphics[width=1.0\columnwidth]{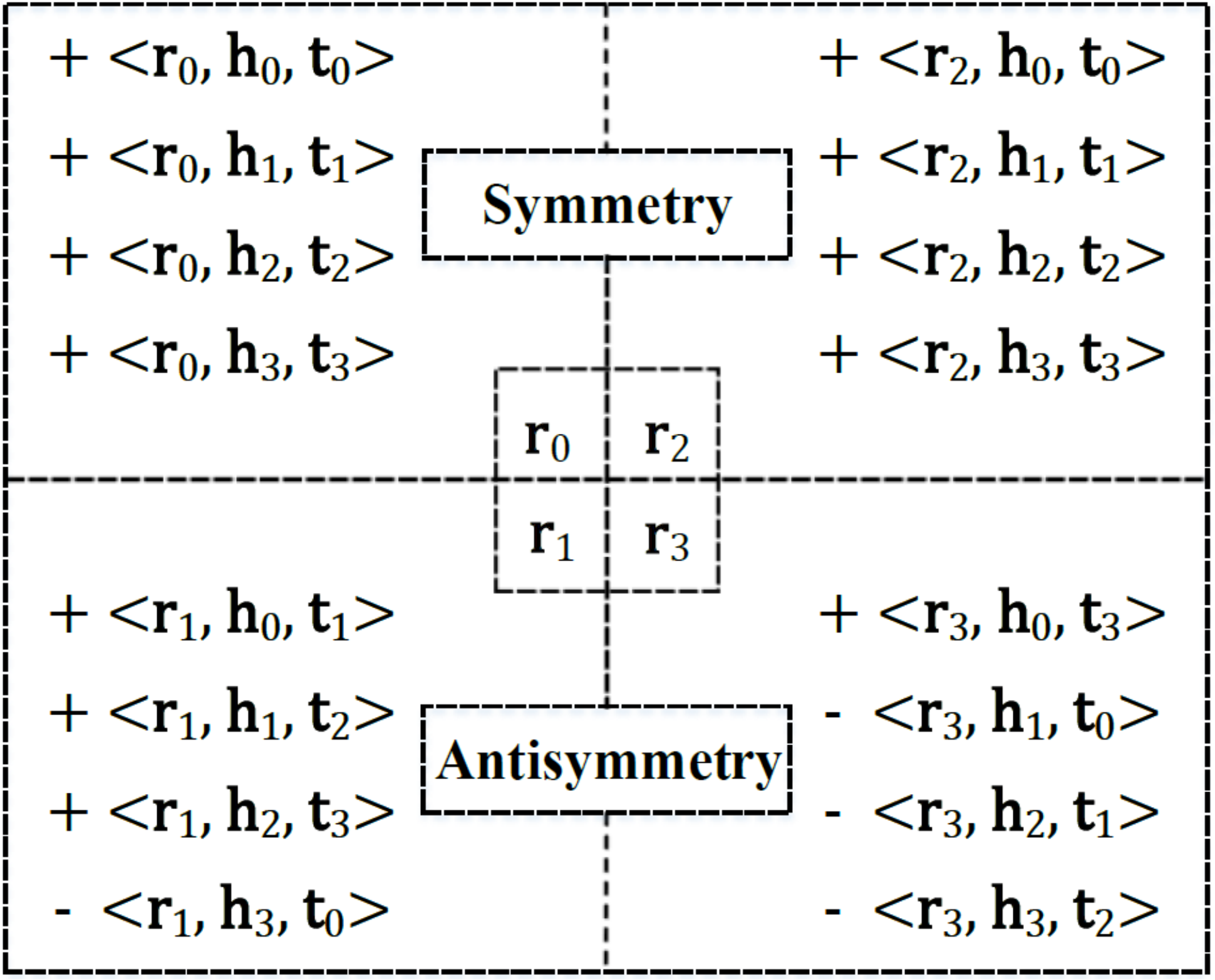}
	\caption{Scoring function $f_4$ with $k=4$.}
	\label{fig:embedding_example}
\end{figure}

The scoring function $f_3$ can support both symmetric and antisymmetric relations inherently because of the design of segmented embeddings.
Moreover, the optimization of relation embeddings is entirely data-driven, and thus we focus on providing the proper mechanism to capture common relation properties.

\subsubsection{$\bm{f_4}$: Reducing Computing Overheads}
However, though capturing various relation properties, the function $f_3$ suffers from huge computation overheads.
The time complexity of function $f_3$ is $O(k^2d)$ because there are $k^3$ dot product terms $\left< \SegVec{r}_x, \SegVec{h}_y, \SegVec{t}_w \right>$ in total.
Therefore, the scoring function $f_3$ needs $k^3$ times of dot product to compute the score of a triple $(h, r, t)$.
Recall that the dimension of each segment is $d/k$, so each multi-linear dot product requires $O(d/k)$ times of multiplication.
As a conclusion, the time complexity of the function $f_3$ is $O(k^2 d)$, which can be calculated by $O(k^3 \times d/k)$.
To reduce the computation overheads of the function $f_3$, we introduce another variable $w_{x,y}$ for the index of tail entity $t$.
Accordingly, we define the scoring function $f_4$ as follows.
\begin{align}
f_4(h, r, t) = \sum_{0 \leq x, y < k} s_{x, y} \cdot \left< \SegVec{r}_x, \SegVec{h}_y, \SegVec{t}_{w_{x,y}} \right>,
\label{eq:seek-score-function}
\end{align}
where
\[
w_{x, y} = 
\left\{
\begin{array}{cl}
y, & \text{if } x \text{ is even}, \\
(x + y) \ \% \ k,  & \text{if } x \text{ is odd}. \\
\end{array}
\right.
\]
The scoring function $f_4$ reduces the number of dot product terms to $k^2$, so its time complexity is $O(kd)$ (calculated by $O(k^2\times d/k)$).
Moreover, the scoring function $f_4$ can also preserve symmetry property in the even parts of $\SegVec{r}$ and preserve antisymmetry property in the odd parts of $\SegVec{r}$. 

Figure~\ref{fig:embedding_example} shows the example of the scoring function $f_4$ with $k=4$. The dot product terms in Figure~\ref{fig:embedding_example} can be categorized into four groups according to the segment indexes of $\SegVec{r}$.  
In the groups of $\SegVec{r}_0$ and $\SegVec{r}_2$, which are the even parts of $\SegVec{r}$, the segment $\SegVec{t}_{w_{x,y}}$'s index $w_{x,y}$ is same as the segment $\SegVec{h}_y$'s index $y$, and $s_{x,y}$ is always positive. Thus, the summation $\sum s_{x, y} \cdot \left< \SegVec{r}_x, \SegVec{h}_y, \SegVec{t}_{w_{x,y}} \right>$ of the even parts of  $f_4(h, r, t)$ is equal to the corresponding one $\sum s_{x, y} \cdot \left< \SegVec{r}_x, \SegVec{t}_y, \SegVec{h}_{w_{x,y}} \right>$ of $f_4(t, r, h)$. 
In the groups of $\SegVec{r}_1$ and $\SegVec{r}_3$, which are the odd parts of $\SegVec{r}$, the segment indexes of $\SegVec{t}$ are $(x + y) \ \% \ k$, where $x$ and $y$ are the indexes of $\SegVec{r}$ and $\SegVec{h}$, respectively. When $x + y \geq k$, the variable $s_{x, y}$ will change from positive to negative. So the summation of the odd parts of $f_4(h, r, t)$ and $f_4(t, r, h)$ will not be the same. Besides, it is apparent that the number of feature interactions on $h$, $r$ and $t$ are increasing $k$ times since each segment has $k$ interactions with other segments.

In summary, the scoring function $f_4$ of our SEEK framework has the following characteristics:
\begin{itemize}[label={-}]

	\item \textit{Tunable Computation.} The scoring function exactly involves each segment of $\SegVec{r}$, $\SegVec{h}$, and $\SegVec{t}$ $k$ times. Thus the number of feature interactions and the computation cost are fully tunable with a single hyperparameter $k$.
	
	\item \textit{Symmetry and Antisymmetry Preservation.}  The even parts of $\SegVec{r}$ can preserve the symmetry property of relations, and the odd parts of $\SegVec{r}$ can preserve the antisymmetry property.
	
	\item \textit{Dimension Isolation.}  The dimensions within the same segment are isolated from each other, which will prevent the embeddings from excessive correlations.
\end{itemize}

\subsection{Discussions}
\paragraph{Complexity analysis}
As described before, the number of dot product terms in scoring function $f_4$ is $k^2$, and each term requires $O(d/k)$ times of multiplication. So the time complexity of our SEEK framework is $O(k d)$ (calculated by $O(k^2\times d/k)$), where $k$ is a small constant such as 4 or 8. For the space complexity, the dimension of entity and relation embeddings is $d$, and there are no other parameters in our SEEK framework. Thus, the space complexity of SEEK is $O(d)$. The low time and space complexity of our framework demonstrate that our SEEK framework has high scalability, which is vital for large-scale real-world knowledge graphs.

\paragraph{Connection with existing methods}
Our SEEK framework is a generalized framework of some existing methods, such as DistMult~\cite{yang2015embedding}, ComplEx~\cite{trouillon2016complex}, and HolE~\cite{nickel2016holographic}. In the following, we will prove that these methods are special cases of our framework when we set $k=1$ and $k=2$, respectively.
\begin{proposition}
	SEEK ($k=1$) is equivalent to DistMult.
\end{proposition}
\begin{proof}
	The proof is trivial. Given $k=1$, we have $x=0$ and $y=0$ in scoring function $f_4$ and $\SegVec{r}_0 = \SegVec{r}$, $\SegVec{h}_0 = \SegVec{h}$, and $\SegVec{t}_0 = \SegVec{t}$. Thus the function $f_4$ can be written as $f_4^{k=1}(h, r, t) = \left< \SegVec{r}, \SegVec{h}, \SegVec{t}\right>$, which is the same scoring function of DistMult.
\end{proof}
\begin{proposition}
	SEEK ($k=2$) is equivalent to the ComplEx and HolE.
\end{proposition}
\begin{proof}
	Given $k=2$, function $f_4$ can be written as:
	$$
	f_4^{k=2}(h, r, t) = \sum_{x=0, 1} \sum_{y=0, 1} s_{x, y} \cdot \left< \SegVec{r}_x, \SegVec{h}_y, \SegVec{t}_{w_{x, y}} \right>,
	$$
	then we expand the right part of the equation:
	$$
	\left< \SegVec{r}_0, \SegVec{h}_0, \SegVec{t}_0 \right> + \left< \SegVec{r}_0, \SegVec{h}_1, \SegVec{t}_1 \right>
	+ \left< \SegVec{r}_1, \SegVec{h}_0, \SegVec{t}_1 \right> - \left< \SegVec{r}_1, \SegVec{h}_1, \SegVec{t}_0 \right>.
	$$
	If we consider $\SegVec{r}_0, \SegVec{h}_0, \SegVec{t}_0$ as the real part of $\SegVec{r}, \SegVec{h}, \SegVec{t}$, and $\SegVec{r}_1, \SegVec{h}_1, \SegVec{t}_1$ as the imaginary part,
	then $f_4^{k=2}(h, r, t)$ is exactly the scoring function of ComplEx framework. Since~\cite{hayashi2017equivalence} has already discussed the equivalence of ComplEx and HolE, the SEEK ($k=2$) is also equivalent to the HolE framework.
\end{proof}

\subsection{Training}
SEEK takes the negative log-likelihood loss function with $L_2$ regularization as its objective function to optimize the parameters of entities and relations:
\begin{align}
\min_{\Theta} \sum_{(h, r, t) \in \Omega}-\log (\sigma (Y_{hrt}f_4(h, r, t))) + \frac{\lambda}{2d} ||\Theta||^2_2,
\label{eq:objective-function}
\end{align}
where $\sigma$ is a sigmoid function defined as $\sigma(x) = \frac{1}{1 + e^{-x}}$, and $\Theta$ represents the parameters in the embeddings of entities and relations in knowledge graphs; $\Omega$ is the triple set containing the true triples in the knowledge graphs and the false triples generated by negative sampling. In the negative sampling, we generate a false triple $(h^{'}, r, t)$ or $(h, r, t^{'})$ by replacing the head or tail entity of a true triple with a random entity. $Y_{hrt}$ is the label of $(h, r, t)$, which is $1$ for the true triples and $-1$ for the false triples. $\lambda$ is the $L_2$ regularization parameter.

The gradients of Equation~\ref{eq:objective-function} are then given by:
\begin{align}
\frac{\partial \mathcal{L}}{\partial \theta} = \frac{\partial \mathcal{L}}{\partial f_4}\frac{\partial f_4}{\partial \theta} + \frac{\lambda \theta}{d},
\end{align}
where $\mathcal{L}$ represents the objective function of SEEK, and $\theta$ is the parameters in the segments. Specifically, the partial derivatives of function $f_4$ for the $x$-th segment of $\SegVec{r}$ and the $y$-th segment of $\SegVec{h}$ are:
$$\frac{\partial f_4}{\partial \SegVec{r}_x} = \sum_{0 \leq y < k} s_{x, y} \cdot (\SegVec{h}_y \odot \SegVec{t}_{w_{x, y}}),$$
$$\frac{\partial f_4}{\partial \SegVec{h}_y} = \sum_{0 \leq x < k} s_{x, y} \cdot (\SegVec{r}_x \odot \SegVec{t}_{w_{x, y}}),$$
where $\odot$ is the entry-wise product of two vectors, e.g. $\SegVec{c} = \SegVec{a} \odot \SegVec{b}$ 
results in the $i$-th dimension of $\SegVec{c}$ is $\SegVec{a}_i \cdot \SegVec{b}_i$.
The derivative of scoring function $f_4$ for $\SegVec{t}_{w}$ is different from those of the above two:
$$\frac{\partial f_4}{\partial \SegVec{t}_w} = \sum_{0 \leq x, y < k} \mathbbm{1}_{[w=w_{x, y}]} \cdot s_{x, y} \cdot (\SegVec{r}_x \odot \SegVec{h}_{y}),$$
where $\mathbbm{1}_{[w=w_{x, y}]}$ has value $1$ if $w=w_{x, y}$ holds, otherwise it is $0$.

\section{Experimental Evaluation}

\label{sec:exp}
In this section, we present thorough empirical studies to evaluate and analyze our proposed SEEK framework. 
We first introduce the experimental setting. Then we evaluate our SEEK framework on the task of link prediction. Then, we study the influence of the number of segments $k$ to the SEEK framework, and present the case studies to demonstrate why our SEEK framework has high effectiveness.

\subsection{Experimental Setting}
\paragraph*{Datasets}
In our experiments, we firstly use a \textit{de facto} benchmark dataset: FB15K. FB15K is a subset of the Freebase dataset~\cite{bollacker2008freebase}, and we used the same training, validation and test set provided by \cite{bordes2013translating}. We also use another two new datasets proposed in recent years: DB100K~\cite{boyang2018:aer} and YAGO37~\cite{guo2018:RUGE}. DB100K was built from the mapping-based objects of core DBpedia~\cite{bizer2009dbpedia}; YAGO37 was extracted from the core facts of YAGO3~\cite{mahdisoltani2013yago3}. Table~\ref{tab:datasets} lists the statistics of the three datasets.

\begin{table}[!h]
	\centering
	\TableSize
	\begin{tabular}{ p{3.4em} | p{2.9em}  p{1.9em}  p{3em}  p{2.5em}  p{2.5em} }
		\toprule
		\textbf{Dataset} & \textbf{\#Ent} & \textbf{\#Rel} & \textbf{\#Train} & \textbf{\#Valid} & \textbf{\#Test} \\
		\midrule
		FB15K & $14,951$ & $1,345$ & $483,142$ & $50,000$ & $59,071$  \\
		DB100K & $99,604$ &  $470$ &  $597,572$ &  $50,000$ &  $50,000$ \\
		YAGO37 & $123,189$ & $37$ & $989,132$ & $50,000$ & $50,000$ \\
		\bottomrule
	\end{tabular}
	\caption{Statistics of datasets.}
	\label{tab:datasets}
\end{table}

\paragraph*{Compared Methods}
There are many knowledge graph embedding methods developed in recent years. We categorize the compared methods as the following groups:
\begin{itemize}[label={-},leftmargin=1.5em]
	\item Some simple knowledge graph embedding methods that have low time and space complexity, like TransE~\cite{bordes2013translating}, DistMult~\cite{yang2015embedding}, HolE~\cite{nickel2016holographic}, ComplEx~\cite{trouillon2016complex}, and Analogy~\cite{liu2017analogical}. Specifically, TransE is a translation based method, and others are the multi-linear dot product-based framework.
	
	\item Some methods that achieve state-of-the-art performance on DB100K and YAGO37, which include RUGE~\cite{guo2018:RUGE} and ComplEx-NNE+AER~\cite{boyang2018:aer}.
	
	\item Some latest methods that achieve current state-of-the-art performance on FB15K, including Single DistMult~\cite{kadlec2017knowledge}, ConvE~\cite{dettmers2018conve}, SimplE~\cite{kazemi2018simple}, RotatE~\cite{sun2018rotate}, and DihEdral~\cite{DBLP:conf/acl/XuL19}. 
	
	\item We evaluate the scoring function $f_2$ to apply an ablation study for our approach. Then we can observe the respective effect of facilitating sufficient feature interactions and preserving the relation properties. Since the scoring function $f_2$ can only preserve the symmetric property, we refer to it as Sym-SEEK.
\end{itemize}
Since our framework does not use additional information like text~\cite{toutanova2015observed}, relational path~\cite{ebisu2019graph}, or external memory~\cite{shen2017modeling}, we do not compare the methods with additional information. Moreover, we only compare our method with single models, and the Ensemble DistMult~\cite{kadlec2017knowledge} is a simple ensemble of multiple different methods, so we do not compare with it.

\begin{table*}[!h]
	\centering
	\begin{threeparttable}
		
		\TableSize
		\begin{tabular}{ l |c  c  c  c | c  c  c  c}
			\toprule
			\multirow{3}{*}{\textbf{Methods}} & \multicolumn{4}{c|}{\textbf{DB100K}} &  \multicolumn{4}{c}{\textbf{YAGO37}} \\
			\cmidrule{2-9}
			& \multirow{2}{*}{\textbf{MRR}}  &  \multicolumn{3}{c|}{\textbf{Hits@N}} & \multirow{2}{*}{\textbf{MRR}}  &  \multicolumn{3}{c}{\textbf{Hits@N}} \\ 
			\cmidrule{3-5}	\cmidrule{7-9}
			&  & \textbf{1} & \textbf{3} & \textbf{10} & & \textbf{1} & \textbf{3} & \textbf{10}  \\ 
			\midrule
			TransE~\cite{bordes2013translating}& $0.111$ & $1.6$ & $16.4$ & $27.0$ & $0.303$ & $21.8$ & $33.6$ & $47.5$   \\
			DistMult~\cite{yang2015embedding}& $0.233$ & $11.5$ & $30.1$ & $44.8$ & $0.365$ & $26.2$ & $41.1$ & $57.5$ \\
			HolE~\cite{nickel2016holographic} & $0.260$ & $18.2$ & $30.9$ & $41.1$ &  $0.380$ & $28.8$ & $42.0$ & $55.1$  \\
			ComplEx~\cite{trouillon2016complex}& $0.242$ & $12.6$ & $31.2$ & $44.0$  & $0.417$ & $32.0$ & $47.1$ & $60.3$ \\		
			Analogy~\cite{liu2017analogical} & $0.252$ & $14.2$ & $32.3$ & $42.7$ & $0.387$ & $30.2$ & $42.6$ & $55.6$ \\
			\midrule
			RUGE~\cite{guo2018:RUGE}& $0.246$ & $12.9$ & $32.5$ & $43.3$ &$0.431$ & $34.0$ & $48.2$ & $60.3$ \\
			ComplEx-NNE+AER~\cite{boyang2018:aer} & $0.306$ & $24.4$ & $33.4$ & $41.8$ & $-$ & $-$ & $-$ & $-$ \\
			\midrule
			\textbf{Sym-SEEK}\tnote{*} & $0.306$ &  $22.5$ & $34.3$ & $46.2$ &$0.452$ &  $36.7$ & $\Outperform{49.8}$ & $60.6$ \\
			\textbf{SEEK}\tnote{*} & $\Outperform{0.338}$ &  $\Outperform{26.8}$ & $\Outperform{37.0}$ & $\Outperform{46.7}$ &$\Outperform{0.454}$ &  $\Outperform{37.0}$ & $\Outperform{49.8}$ & $\Outperform{62.2}$ \\
			\bottomrule
		\end{tabular}
		\begin{tablenotes}[para,flushleft]
			\item[*] Statistically significant improvements by independent $t$-test with $p = 0.01$.
		\end{tablenotes}
		
	\end{threeparttable}
	\caption{Results of link prediction on DB100K and YAGO37.}
	\label{tab:kbc-performance-yago37-db100k}
\end{table*}

\paragraph*{Experimental Details}

We use the asynchronous stochastic gradient descent (SGD) with the learning rate adapted by AdaGrad~\cite{duchi2011adaptive} to optimize our framework. The loss function of our SEEK framework is given by Equation \ref{eq:objective-function}. We conducted a grid search to find hypeparameters which maximize the results on validation set, by tuning number of segments $k\in\{1, 2, 4, 8, 16, 20\}$, the dimension of embeddings $D\in\{100, 200, 300, 400\}$, $L_2$ regularization parameter $\lambda\in\{0.1, 0.01, 0.001, 0.0001\}$ and the number of negative samples per true triple $\eta\in\{10, 50, 100, 500, 1000\}$. The optimal combinations of hyperparameters are $k = 8$, $D = 400$, $\lambda = 0.001$, $\eta = 1000$ on FB15K; $k = 4$, $D = 400$, $\lambda = 0.01$, $\eta = 100$ on DB100K; and $k = 4$, $D = 400$, $\lambda = 0.001$, $\eta = 200$ on YAGO37. We set the initial learning rate $lr$ to 0.1 and the number of epochs to 100 for all datasets. 

\subsection{Link Prediction}

We study the performance of our method on the task of link prediction, which is a prevalent task to evaluate the performance of knowledge graph embeddings. We used the same data preparation process as \cite{bordes2013translating}. Specifically, we replace the head/tail entity of a true triple in the test set with other entities in the dataset and name these derived triples as \textit{corrupted triples}. The goal of the link prediction task is to score the original true triples higher than the corrupted ones. We rank the triples by the results of the scoring function.

\begin{table}[!h]
	\centering
	\begin{threeparttable}
		
		\TableSize
		\begin{tabular}{ l | c  c  c  c}
			\toprule
			\multirow{3}{*}{\textbf{Methods }}  &  \multicolumn{4}{c}{\textbf{FB15K}} \\
			\cmidrule{2-5}	
			& \multirow{2}{*}{\textbf{MRR}}  &  \multicolumn{3}{c}{\textbf{Hits@N}} \\ 
			\cmidrule{3-5}	
			& & \textbf{1} & \textbf{3} & \textbf{10}  \\ 
			\midrule
			TransE& $0.380$ & $23.1$ & $47.2$ & $47.1$ \\
			DistMult& $0.654$ & $54.6$ & $73.3$ & $72.8$ \\
			HolE & $0.524$ & $40.2$ & $61.3$ & $73.9$ \\
			ComplEx& $0.692$ & $59.9$ & $75.9$ & $84.0$ \\		
			Analogy& $0.725$ & $64.6$ & $78.5$ & $85.4$ \\
			\midrule
			RUGE& $0.768$ & $70.3$ & $81.5$ & $86.5$ \\
			ComplEx-NNE+AER & $0.803$ & $76.1$ & $83.1$ & $87.4$\\
			\midrule
			Single DistMult & $0.798$ &$-$ &$-$&$\Outperform{89.3}$ \\	
			ConvE& $0.745$ & $67.0$ & $80.1$ & $87.3$ \\
			SimplE& $0.727$ & $66.0$ & $77.3$ & $83.8$  \\
			RotatE& $0.797$ & $74.6$ & $83.0$ & $88.4$ \\ 
			DihEdral & $0.733$ & $64.1$ & $80.3$ & $87.7$ \\
			\midrule
			\textbf{Sym-SEEK}\tnote{*} & $0.796$ &  $74.7$ & $82.9$ & $88.2$ \\
			\textbf{SEEK}\tnote{*} & $\Outperform{0.825}$ &  $\Outperform{79.2}$ & $\Outperform{84.1}$ & $88.6$ \\
			\bottomrule
		\end{tabular}
		
		\begin{tablenotes}[para,flushleft]
			\item[*] Statistically significant improvements by independent $t$-test with $p = 0.01$.
		\end{tablenotes}
		
	\end{threeparttable}
	\caption{Results of link prediction on FB15K.}	
	\label{tab:kbc-performance-fb15k}
\end{table}

We use the MRR and Hit@N metrics to evaluate the ranking results:
\begin{enumerate*}[label={\alph*)}]
	\item MRR: the mean reciprocal rank of original triples;
	\item Hits@N: the percentage rate of original triples ranked at the top $n$ in prediction.
\end{enumerate*}
For both metrics, we remove some of the corrupted triples that exist in datasets from the ranking results, which is also called \emph{filtered} setting in~\cite{bordes2013translating}. We use Hits@1, Hits@3, and Hits@10 for the metrics of Hits@N.

Table~\ref{tab:kbc-performance-yago37-db100k} summarizes the results of link prediction on DB100K and YAGO37, and Table~\ref{tab:kbc-performance-fb15k} shows the results on FB15K. Note, the results of compared methods on DB100K and YAGO37 are taken from~\cite{boyang2018:aer,guo2018:RUGE}; the results on FB15K are taken from~\cite{kadlec2017knowledge,boyang2018:aer,kazemi2018simple,sun2018rotate,DBLP:conf/acl/XuL19}.

\begin{figure}[!h]
	\centering
	\includegraphics[width=0.96\columnwidth]{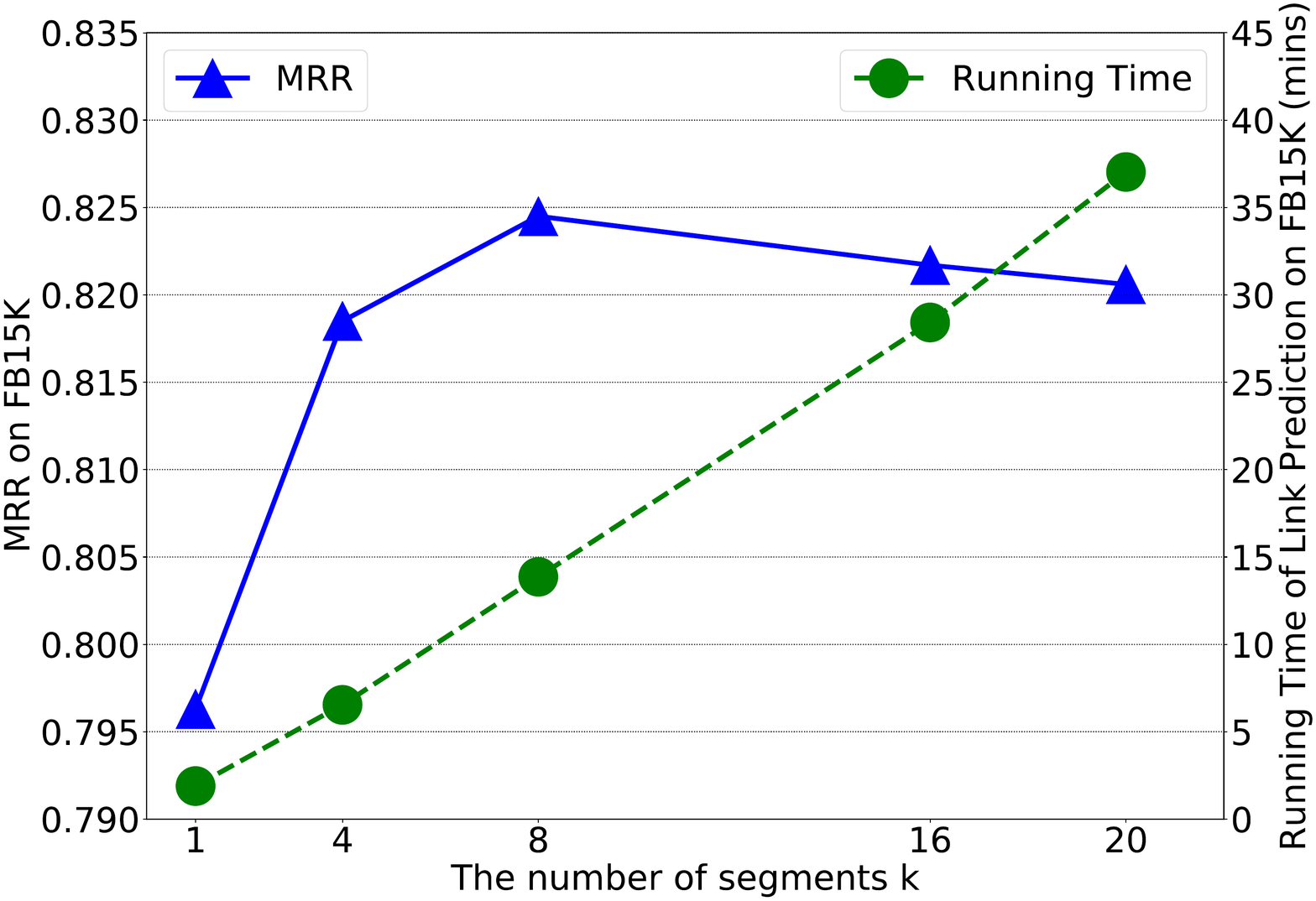}
	\caption{The influence of the number of segments $k$ to the MRR and the running time of link prediction on FB15K.}
	\label{fig:Seg_count}
\end{figure}

On the DB100K, SEEK outperforms the compared methods in all metrics, and the Sym-SEEK also can achieve a good performance. On the YAGO37, the SEEK and Sym-SEEK have a similar result and outperform other previous methods. The results on YAGO37 show that exploiting more feature interactions can significantly improve the performance of the embeddings on YAGO37 while preserving the semantic properties have a slight improvement. On FB15K, SEEK achieves the best performance on MRR, Hit@1 and Hit@3. 
Although SEEK is worse than the Single DistMult on the metrics of Hit@10, the Single DistMult is just a higher dimensional version of DistMult. The Single DistMult uses 512-dimensional embeddings, which is larger than the 400-dimensional embeddings of the SEEK framework. 
On the whole, our method's improvements on these datasets demonstrate that our method has a higher expressiveness.

\subsection{Influence of the Number of Segments $k$}
In the SEEK framework, a larger number of segments $k$ implies more feature interactions and higher computational cost. To empirically study the influence of the number of segments $k$ to the performance and computation time of SEEK, we let $k$ vary in $\{1, 4, 8, 16, 20\}$ and fix all the other hyperparameters, then we observe the MRR and time costs for the link prediction task on the test set of FB15K. 

\begin{figure}[!h]
	\centering
	\includegraphics[width=1.0\columnwidth]{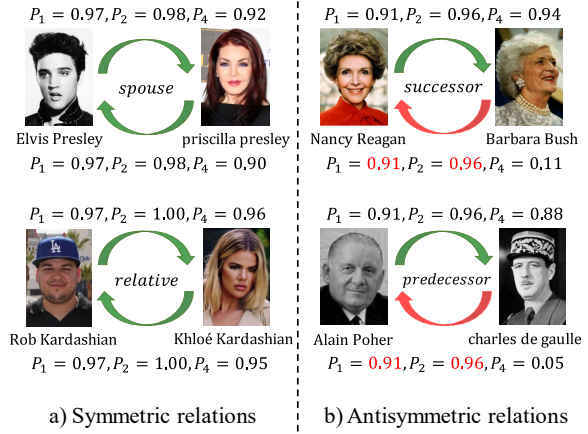}
	\caption{The correct probabilities of four triples in DB100K and their reverse triples. The probabilities $P_1$, $P_2$ and $P_4$ are corresponding to the scoring functions $f_1$, $f_2$ and $f_4$, respectively.}
	\label{fig:case_study}
\end{figure}
Figure~\ref{fig:Seg_count} shows the MRR and time costs of different segment counts $k$ on FB15K. As we can see, changing $k$ affects the performance of knowledge graph embeddings significantly. When $k$ varies from $1$ to $8$, the performance is increased steadily. However, when $k$ becomes even larger, no consistent and dramatic improvements observed on the FB15K dataset. This phenomenon suggests that excessive feature interactions cannot further improve performance. Therefore, $k$ is a sensitive hyperparameter that needs to be tuned for the best performance given a dataset. Figure~\ref{fig:Seg_count} also illustrates that the running time of SEEK is linear in $k$, and it verifies that the time complexity of SEEK is $O(kd)$.

\subsection{Case Studies}
We employ case studies to explain why our framework has a high expressiveness.
Specifically, we utilize the scoring functions $f_1$, $f_2$ and $f_4$ to train the embeddings of DB100K, respectively. Then we use the corresponding scoring functions to score the triples in the test set and their reverse triples, and we feed the scores to the sigmoid function to get the correct probabilities $P_1$, $P_2$ and $P_4$ of each triple.
Figure~\ref{fig:case_study} shows the correct probabilities of some triples. In these triples, two triples have symmetric relations, and the other two have antisymmetric relations. On the triples with symmetric relations, the original triples in the test set and their reverse triples are true triples, and the scoring functions $f_1$, $f_2$, $f_4$ can result in high probabilities on original and reverse triples. On the triples with antisymmetric relations, the reverse triples are false. Since the values of $f_1(h, r, t)$ or $f_2(h, r, t)$ are equal to $f_1(t, r, h)$ or $f_2(t, r, h)$, the scoring functions $f_1$ and $f_2$ result in high probabilities on the reverse triples. But the scoring function $f_4$, which can model both symmetric and antisymmetric relations, results in low probabilities on the reverse triples. Meanwhile, we can also find that function $f_2$ have higher probabilities than function $f_1$ on the true triples. This phenomenon further explains that facilitating sufficient feature interactions can improve the expressiveness of embeddings.

\section{Conclusion and Future Work}
\label{sec:conclusion}
In this paper, we propose a lightweight KGE framework (SEEK) that can improve the expressiveness of embeddings without increasing the model complexity. To this end, our framework focuses on designing scoring functions and highlights two critical characteristics: 1) facilitating sufficient feature interactions and 2) preserving various relation properties.
Besides, as a general framework, SEEK can incorporate many existing models, such as DistMult, ComplEx, and HolE, as special cases.
Our extensive experiments on widely used public benchmarks demonstrate the efficiency, the effectiveness, and the robustness of SEEK.
In the future, we plan to extend the key insights of segmenting features and facilitating interactions to other representation learning problems.
\section*{Acknowledgments}

This work is supported by the National Natural Science Foundation of China (U1711262, U1611264, U1711261, U1811261, U1811264, U1911203), National Key R\&D Program of China (2018YFB1004404), Guangdong Basic and Applied Basic Research Foundation (2019B1515130001),  Key R\&D Program of Guangdong Province (2018B010107005).

\balance
\bibliography{acl2020}
\bibliographystyle{acl_natbib}

\end{document}